\documentclass[10pt,twocolumn,letterpaper]{article}

\usepackage[pagenumbers]{cvpr} 

\usepackage{amsmath,amsfonts,bm}

\def\1{\bm{1}}

\def\rc{{\textnormal{c}}}

\def\rvx{{\mathbf{x}}}
\def\rvy{{\mathbf{y}}}

\DeclareMathAlphabet{\mathsfit}{\encodingdefault}{\sfdefault}{m}{sl}
\SetMathAlphabet{\mathsfit}{bold}{\encodingdefault}{\sfdefault}{bx}{n}

\newcommand{\E}{\mathbb{E}}
\newcommand{\Ls}{\mathcal{L}}
\newcommand{\R}{\mathbb{R}}

\definecolor{cvprblue}{rgb}{0.21,0.49,0.74}
\usepackage[pagebackref,breaklinks,colorlinks,allcolors=cvprblue]{hyperref}
\usepackage{amsthm}
\usepackage{amsmath}
\usepackage{booktabs}
\usepackage{multirow}
\usepackage{xspace}
\usepackage{subcaption}

\usepackage[table]{xcolor}
\usepackage{algorithm}
\usepackage{multicol}
\usepackage{algpseudocode}
\usepackage{mathtools}

\algrenewcommand{\algorithmiccomment}[1]{\hfill\textcolor{green!50!black}{$\triangleright$ #1}}
\def\NoNumber#1{{\def\alglinenumber##1{}\State #1}\addtocounter{ALG@line}{-1}}

\def\paperID{10040} 
\def\confName{CVPR}
\def\confYear{2026}

\title{Advancing Image Classification with Discrete Diffusion Classification Modeling}

\author{
Omer Belhasin, Shelly Golan, Ran El-Yaniv, and Michael Elad\\
Technion - Israel Institute of Technology\\
{\tt\small \{omer.be,shelly.golan,rani,elad\}@cs.technion.ac.il}
}

\begin{document}
\maketitle

\begin{abstract}
Image classification is a well-studied task in computer vision,
and yet it remains challenging under high-uncertainty conditions, such as when input images are corrupted or training data are limited.
Conventional classification approaches typically train models to directly predict class labels from input images, but this might lead to suboptimal performance in such scenarios.
To address this issue, we propose Discrete Diffusion Classification Modeling (DiDiCM), a novel framework that leverages a diffusion-based procedure to model the posterior distribution of class labels conditioned on the input image.
DiDiCM supports diffusion-based predictions either on class probabilities or on discrete class labels, providing flexibility in computation and memory trade-offs.
We conduct a comprehensive empirical study demonstrating the superior performance of DiDiCM over standard classifiers, showing that a few diffusion iterations achieve higher classification accuracy on the ImageNet dataset compared to baselines, with accuracy gains increasing as the task becomes more challenging. We release our code at \url{https://github.com/omerb01/didicm}.
\end{abstract}

\section{Introduction}
\label{sec:intro}

\begin{figure}[t]
    \centering
    \includegraphics[width=0.8\columnwidth]{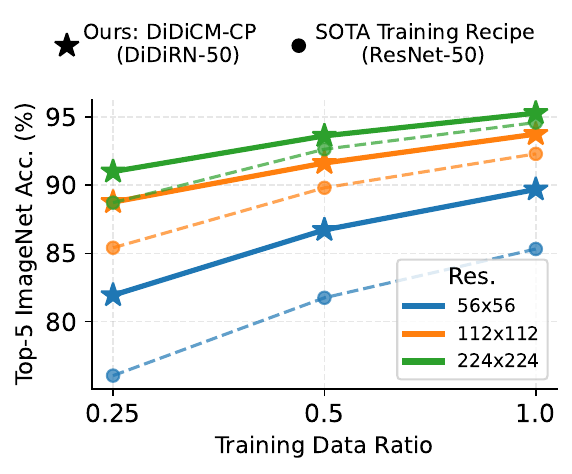}
    \caption{\textbf{ImageNet Top-5 Accuracy:} DiDiCM vs. standard classifiers. DiDiRN-50 (comparable to ResNet-50) and ResNet-50 are both trained using the state-of-the-art recipe \cite{wightman2021resnet}. DiDiCM shows superior top-5 accuracy across all uncertainty settings.}
    \label{fig:main_figure}
\end{figure}

Image classification has long been a central problem in computer vision and has driven major advances in deep learning \citep{lecun2002gradient,krizhevsky2012imagenet,russakovsky2015imagenet,he2015delving,he2016deep,dosovitskiy2020image}. During the past decade, this task has been approached through various model architectures \cite{tan2019efficientnet,radford2021learning,ding2022davit,yu2022coca,woo2023convnext} and training paradigms \cite{hinton2015distilling,chen2020simple,grill2020bootstrap,wightman2021resnet,he2022masked}. Despite these advances, classification remains challenging in real-world scenarios characterized by high uncertainty, where data are noisy, incomplete, or limited, such as in medical imaging \citep{tajbakhsh2020embracing,zhou2021review}, autonomous driving \citep{grigorescu2020survey,caesar2020nuscenes}, and other domains \cite{zhu2017deep,norouzzadeh2018automatically}. These conditions motivate the development of alternative modeling strategies that are capable of explicitly handling uncertainty in the data.

Uncertainty in classification is a fundamental challenge that has been defined and addressed in various ways \citep{lakshminarayanan2017simple,geifman2017selective,goan2020bayesian,romano2020classification}.
A common thread among these approaches is Bayesian modeling of the posterior distribution of class labels conditioned on an input image, which captures the confidence and ambiguity of the model across class labels.

Yet, how uncertainty propagates throughout the learning process remains an open and intriguing question.
Traditional classifiers are typically trained to predict class labels using the standard cross-entropy loss.
However, this approach introduces stochastic biases and degrades performance, particularly in high-uncertainty settings such as corrupted images or low-data regimes \cite{clarte2023theoretical}.
A more desirable approach would refer to the given data as stochastic samples from the target distribution, and optimize over it accordingly.

Recently, diffusion-based generative models have revolutionized image synthesis \cite{sohl2015deep,song2019generative,ho2020denoising}, surpassing the state-of-the-art performance previously achieved by GANs \cite{dhariwal2021diffusion}.
These models have shown a remarkable ability to model complex data distributions in continuous spaces.
Motivated by their success, we explore whether diffusion mechanisms can be adapted to the discrete domain to enhance image classification performance.

To this end, we introduce \emph{Discrete Diffusion Classification Modeling} (DiDiCM), a novel diffusion framework that models the posterior distribution of class labels conditioned on input images.
Our approach simulates the reverse diffusion process by taking advantage of the tractability of class labels, operating either on class probabilities or on discrete class labels, thereby offering flexibility in computation and memory trade-offs.
At its core, the model predicts the Concrete Score \citep{meng2022concrete,lou2023discrete}, a generalization of the standard continuous score function \cite{song2019generative}, enabling diffusion directly in the discrete label space.

To the best of our knowledge, DiDiCM is the first discrete diffusion-based framework specifically designed for classification. In contrast to previous approaches that adapt diffusion models originally developed for image generation to classification tasks \cite{zimmermann2021score,li2023your,guo2023egc,daultani2024diffusion,yadin2024classification}, our method directly approximates the target posterior distribution over labels. Existing works on discrete diffusion processes have primarily focused on the language domain and suffer from significant tractability challenges, limitations that our framework effectively overcomes to achieve improved performance.

To examine the suggested diffusion framework, we propose the \emph{Discrete Diffusion Residual Network} (DiDiRN), a DiDiCM-compatible classification architecture extending the well-known ResNet \cite{he2016deep}.
DiDiRN preserves the core convolutional structure of ResNet while introducing additional modules that enable the processing of the current diffusion step.
This design allows for a direct and fair comparison with standard ResNet architectures.

We conduct extensive experiments on the ImageNet dataset under controlled uncertainty settings, involving varying levels of image corruption and data scarcity. Our training follows the state-of-the-art methodology proposed by \citet{wightman2021resnet}, a comprehensive recipe that integrates diverse data augmentations and optimization techniques to enhance model performance, thereby establishing a strong and competitive baseline. In Figure~\ref{fig:main_figure}, we present the top-5 accuracy results on the ImageNet dataset, demonstrating that DiDiCM, when applied with DiDiRN, consistently outperforms ResNet classifiers, achieving higher accuracy with only a few diffusion steps. Moreover, the performance gap widens as uncertainty increases,
highlighting DiDiCM’s success in approximating the posterior distribution compared to standard classification.

In summary, the following are the four main contributions of this paper:  
\begin{enumerate}
\item We propose a novel diffusion-based classification framework that defines the forward and reverse processes along with the corresponding training objective; 
\item We introduce two approaches to simulate the reverse process, either via diffusion of class probabilities or discrete class labels, offering flexibility in the computational complexity versus memory tradeoff;
\item We develop a diffusion-based classification architecture built on ResNet, enabling fair comparison with well-established ResNet performance in the literature;
and 
\item We provide a thorough empirical analysis of our method under varying uncertainty levels on the challenging ImageNet dataset.
\end{enumerate}

\section{Problem Formulation}
\label{sec:formulation}

Let $\rvx$ be a random vector, and let $\rvy = h(\rvx)$ be an observation of $\rvx$, where $h$ is an unknown, possibly stochastic and non-invertible function.
For example, $\rvx$ could represent a clear image, while $\rvy$ could be its low-resolution version.
Note that $h$ may also be the identity function, in which case $\rvy = \rvx$.
Let $\rc \in \{ 1, \dots, K \}$ be a random class label assigned to $\rvx$, where $\rc$ lives in a finite discrete space of $K$ class labels.
In this paper, our goal is to model $P(\rc|\rvy) \in \R^K$ with a parameterized model $p_\theta(\rc|\rvy) \approx P(\rc|\rvy)$.

One might be tempted to adopt a common yet naive approach to the classification task by training a model $f_\theta(\rvy) \in \R^K$ to predict the class label from the corrupted input using the standard cross-entropy loss, $\ell_\text{CE}(\rvy,\rc) := -\log f_\theta(\rvy)_{\rc}$.
However, this approach 
disregards the inherent uncertainty by the transformation from $\rvx$ to $\rvy$.
In particular, it optimizes $\ell_\text{CE}(\rvy,\rc)$ over $\rc \sim P(\rc|\rvx)$ rather than $\rc \sim P(\rc|\rvy)$, thereby introducing unavoidable stochastic biases into the optimization process and leading to suboptimal performance.
This problem becomes even more prominent when training data is limited \cite{clarte2023theoretical}, which is common in various classification domains where class labels might be rare.

In an effort to address the classification challenge in high-uncertainty settings, we propose an alternative classification framework based on a discrete diffusion process for modeling the distribution of class labels given an input image. We start by reviewing existing and related work in the next section, and then dive into the proposed scheme.


\section{Related Work}
\label{sec:related_work}

\textbf{Discrete Diffusion Models}. ~
Most discrete diffusion models have been proposed for the language domain and follow the framework introduced by D3PM \cite{austin2021structured}, which employs a discrete-time formulation.
In this framework, the model is trained to denoise a noisy instance back to its clean form using a weighted cross-entropy loss.
As discussed in \cite{clarte2023theoretical}, this approach inherits the same limitations as conventional classifiers in the classification domain.
A generalization of D3PM is SEDD \cite{lou2023discrete}, a discrete diffusion language model that applies score matching \cite{song2019generative} through a continuous-time discrete diffusion framework \cite{campbell2022continuous}. SEDD builds upon the Concrete Score \cite{meng2022concrete}, which extends the standard score function in continuous diffusion modeling \citep{song2019generative}.
In our paper, we build upon the principles of SEDD and adapt them to the classification domain.
By leveraging the tractability of class labels, we reformulate the diffusion processes into an efficient and high-performing discrete diffusion framework specifically tailored for classification tasks.

\textbf{Diffusion-based Classification}. ~
To the best of our knowledge, this work is the first to apply a discrete diffusion approach to image classification.  
Previous studies on diffusion-based classification have adapted diffusion models originally developed for image generation tasks \cite{zimmermann2021score,li2023your,guo2023egc,daultani2024diffusion,yadin2024classification} to serve classification objectives.  
These methods typically demand extensive training datasets and high amount of computational resources for both training and inference.  
In contrast, our method operates directly in the class label domain; it is more analytically justified, and it offers a more efficient diffusion-based classification framework.

\textbf{Uncertainty-Aware Classification}. ~
Prior research on uncertainty in classification tasks has primarily addressed uncertainty arising from (1) data-related limitations, (2) the inherent ambiguity of the classification task, or (3) their combination \cite{einbinder2022training,jones2022direct,zhou2022survey,dawood2023uncertainty,sale2024label,belhasin2024uncertainty}. Recently, \citet{belhasin2024uncertainty} introduced an uncertainty-aware, diffusion-based approach for handling  inverse problems in cardiovascular diagnosis. This method was shown to achieve consistent performance gains in diagnosis, through the analytically grounded \emph{Expected Score Classifier} (ESC), which models the posterior distribution of class labels given degraded inputs. The ESC reconstructs clean signals before classification, effectively averaging predictions across degraded observations. Building on this idea, our work extends uncertainty modeling to image classification by directly estimating the posterior distribution of class labels conditioned on corrupted images, avoiding the need for explicit clean-image reconstruction and therefore dramatically reducing computational complexity.

\section{Discrete Diffusion Classification Modeling}
\label{sec:didicm}

In this section, we propose \emph{Discrete Diffusion Classification Modeling} (DiDiCM), the diffusion-based framework that underlies our work. 
Building on \cite{anderson2012continuous,campbell2022continuous,meng2022concrete,lou2023discrete}, we address the classification task within a continuous-time diffusion framework, where the target distribution of interest is the posterior, $P(\rc|\rvy)$. 
We then reformulate the diffusion processes to effectively model the posterior in a probabilistic manner.


\subsection{Forward Process}
\label{sec:forward}

We define the forward process across $t \in [0,1]$ as an evolution of noisy distributions $q(\rc_t|\rvy) \in \R^K$ according to a continuous time Markov process given by a linear ordinary differential equation:
\begin{align}
    \frac{dq(\rc_t|\rvy)}{dt} = R_t \cdot q(\rc_t|\rvy) ~~ \text{s.t.} ~~ R_t := \sigma_t (\mathbf{1} \mathbf{1}^T - KI) ~,
\end{align} 
where $q(\rc_0|\rvy) \in \R^K$ approximates $P(\rc|\rvy)$, and $\sigma_t \in [0,\infty)$ is a strictly decreasing function satisfying $\sigma_1 = 0$.

The matrix $R_t = \sigma_t R$ is the uniform transition \emph{rate} matrix \cite{anderson2012continuous,campbell2022continuous}.
At each diffusion step $t$, this matrix transforms the class label associated with the input into other random classes with some probability.
More generally, for any $i \neq j$, the 
entry $R_t(i,j)$ is the rate (occurrences per unit time) at which class label $i$ transitions to class label $j$, implying that 
the higher the rate is, the less time it will take for this transition to occur.
During the forward diffusion process, $q(\rc_1|\rvy) \in \R^K$ becomes the uniform distribution over all class labels, corresponding to the completely noisy state. Supplementary theoretical details regarding these rate matrices are provided in Appendix~\ref{app:theoretical_analysis}, with further intuition discussed in Appendix~\ref{app:forward}.

Given rate matrices of the form $R_t = \sigma_t R$, we derive their eigendecomposition using 
$R = U \Lambda U^{-1}$, where $U \in \mathbb{R}^{K \times K}$ denotes the matrix of eigenvectors 
and $\Lambda \in \mathbb{R}^{K \times K}$ is the diagonal matrix of eigenvalues. 
Starting from $q(\rc_0 | \rvy)$, being a one-hot encoding of 
$\rc_0 \sim P(\rc | \rvy)$ representing a sample from the target distribution\footnote{
When training with Mixup~\cite{zhang2017mixup} or CutMix~\cite{yun2019cutmix}, the corresponding mixed distributions are used.},
the sequence of transitions can be accumulated to efficiently compute the forward process 
for any noise level $t \in [0, 1]$ in closed form:
\begin{align}
\label{eq:efficent_forward}
    q(\rc_t|\rvy) = U \exp \left( \overline{\sigma}_t \Lambda \right) U^{-1} \cdot q(\rc_0|\rvy) ~,
\end{align}
where $\overline{\sigma}_t := \int_0^t \sigma_s \, ds \in \mathbb{R}$ denotes the total noise applied to $q(\rc_0|\rvy)$.  
A detailed proof of the above is provided in Theorem~\ref{the:total_forward} in the Appendix.  

Note that the total noise term $\overline{\sigma}_t$ admits a closed-form solution for suitable choices of the noise schedule $\sigma_t$.  
For example, under a log-linear schedule defined by $\sigma_t := a b^t \log b$, with $a, b \in \mathbb{R}$ chosen so that $\sigma_1 \approx 0$, the total noise becomes $\overline{\sigma}_t = a b^t - a$.


\subsection{Reverse Process}
\label{sec:reverse}

The distribution of interest emerges from the reversal of the forward process. We denote these reversal noisy distributions as $p(\rc_t|\rvy) \in \R^K$.

Assume we have access to a score matrix $S_t \in \R^{K \times K}$ s.t. $S_t(i,j;\rvy):= q(\rc_t=i|\rvy) /  q(\rc_t=j|\rvy) \in \mathbb{R}$.
These ratios are the conditional analogue of the Concrete Score \cite{meng2022concrete,lou2023discrete}, which generalizes the standard score function \citep{song2019generative} when applied at diffusion step $t$ for $i \neq j$ (when $i=j$ this ratio is simply $1$).
Each column $j$ in $S_t$ represents the transition scores of the class label $j$.

Following the continuous-time formulations in \citet{anderson2012continuous} and  \citet{campbell2022continuous}, it turns out that the reversal process is defined by another linear ordinary differential equation:
\begin{align}
\label{eq:reverse ode}
    \nonumber
    \frac{dp(\rc_{1-t}|\rvy)}{dt}& = \overline{R}_{1-t} \cdot p(\rc_{1-t}|\rvy) \\
    \text{s.t.} ~~~ \overline{R}_t& := S_t \odot R_t - \operatorname{diag} \left( \mathbf{1}^T \left( S_t \odot R_t \right) \right) ,
\end{align}
where $p(\rc_1|\rvy) := \mathcal{U}(\{1,\dots,K\}) \in \R^K$ is the uniform distribution across class labels, representing the completely noisy state.

We further define the infinitesimal transition matrix $\overline{Q}_t := I+\overline{R}_t \Delta t \in \R^{K \times K}$ for a sufficiently small time increment $\Delta t > 0$.
The reverse process can then be simulated by taking small Euler steps with size $\Delta t$ as follows:
\begin{align}
\label{eq: reverse euler}
    \nonumber
    p(&\rc_{t-\Delta t}|\rvy) \approx \overline{Q}_t \cdot p(\rc_t|\rvy) \\
    &\text{s.t.} ~~~ \overline{Q}_t(i,j;\rvy) =
    \begin{cases}
        S_t(i,j;\rvy) \sigma_t \Delta t & i \neq j \\
        1 - \sum_{\rc \neq j} S_t(\rc,j;\rvy) \sigma_t \Delta t  & i = j 
    \end{cases} ~.
\end{align}
Theorem~\ref{the:matrix-euler} in the Appendix provides a formal proof of this discrete approximation, while additional insights into the structure and intuition behind the transition matrix $\overline{Q}_t$ are presented in Appendix~\ref{app:reverse}.

\subsection{Training Objective}

In practice, the score matrix $S_t \in \mathbb{R}^{K \times K}$ is not directly accessible. 
Following \citet{lou2023discrete}, we train a parameterized model $s_\theta$ to approximate columns in $S_t$, i.e.,
$s_\theta(\rvy, \rc_t, t) \approx [S_t(1, \rc_t; \rvy), \dots, S_t(K, \rc_t; \rvy)]^T \in \mathbb{R}^{K \times 1}$,
where by construction $s_\theta(\rvy, \rc_t, t)_{\rc_t} = 1$.

To train this model, we propose a variant of the score entropy (SE) loss \citep{lou2023discrete}, conditioned on the input $\rvy$ and weighted by the noise level $\sigma_t$. 
We denote this objective as the DiDiCM loss:
\begin{align}
\label{eq:loss}
    \nonumber
    \Ls_\text{DiDiCM}(\theta) &:=
    \mathop{\E}\limits_{\substack{t \sim \mathcal{U}([0,1]) \\[1pt] \rvy, \rc_t \sim q(\rvy, \rc_t)}}
    \Big[~ \frac{\sigma_t}{K} \Big( \mathbf{1}^T A(S_t(\cdot,\rc_t;\rvy))
    \\ + \mathbf{1}^T s_\theta(\rvy&,\rc_t,t) - S_t(\cdot,\rc_t;\rvy)^T \log s_\theta(\rvy,\rc_t,t) \Big) ~\Big] ~,
\end{align}
where $S_t(\cdot, \rc_t; \rvy) \in \mathbb{R}^K$ denotes the column vector corresponding to index $\rc_t$, 
and $A(a) = a(\log a - 1)$ is applied element-wise to ensure that $\mathcal{L}_{\text{DiDiCM}} \ge 0$.
Intuitively, this loss performs score matching \cite{song2019generative} by optimizing $s_\theta(\rvy,\rc_t,t)$ toward $S_t(\cdot,\rc_t;\rvy)$, while enforcing positive-valued scores.

Consequently, given a data pair $(\rvy, \rc_0) \sim P(\rvy, \rc)$, evaluating this loss requires constructing the noisy label distribution $q(\rc_t|\rvy) \in \R^K$ 
via the efficient forward process defined in Equation~\eqref{eq:efficent_forward}. 
We then sample a noisy class label $j \sim q(\rc_t|\rvy)$ and compute the ratios for all $1 \leq i \leq K$ by $S_t(i, j; \rvy) = q(\rc_t = i|\rvy)/q(\rc_t = j|\rvy)$,
which are then substituted into the objective above in matrix-form.
Algorithm~\ref{alg:training} summarizes the training procedure.

Note that while \citet{lou2023discrete} apply a related denoising approach in the language domain, where the corresponding loss of Equation~\eqref{eq:loss} becomes intractable due to the summation over large discrete spaces, our formulation remains fully tractable for classification.

\begin{algorithm}[h]
\caption{One Stochastic Training Step for DiDiCM}
\label{alg:training}
\begin{algorithmic}[1]
\Require
\Statex - Data-label pair $(\rvy, \rc_0) \sim P(\rvy,\rc_0)$.
\Statex - Forward diffusion noise schedule $\{\sigma_t\}_{t\in[0,1]}$.
\Statex - Uniform transition rate matrix $R = \mathbf{1}\mathbf{1}^T - KI$, and its eigendecomposition $R = U \Lambda U^{-1}$.
\Statex - Scoring model $s_\theta$ to be learned.
\State $q_0 \leftarrow q(\rc_0|\rvy) \in \R^K$
\Comment{Init clean dist. (one-hot or mix)}
\State Sample $t \sim \mathcal{U}([0,1])$
\Comment{Sample noise level}
\State $q_t \leftarrow U \exp \left( \overline{\sigma}_t \Lambda \right) U^{-1} \cdot q_0 \in \R^K$
\Comment{Apply forward}
\State Sample $j \sim q_t$
\State $s \leftarrow q_t / [q_t]_j \in \R^K$ 
\State $\hat{s} \leftarrow s_\theta(\rvy, j, t) \in \R^K$
\Comment{Predict the Concrete Score}
\State $\hat{\mathcal{L}}_\text{DiDiCM}(\theta) = \frac{\sigma_t}{K} \sum_{i=1}^K ( \hat{s}_i - s_i \log \hat{s}_i + A(s_i) )$
\NoNumber{\Comment{Compute DiDiCM loss}}
\State Backpropagate on $\nabla_{\!\theta} \hat{\mathcal{L}}_\text{DiDiCM}(\theta)$ and run optimizer
\end{algorithmic}
\end{algorithm}

\section{Simulating DiDiCM with Concrete Scores}

\begin{figure*}
\begin{multicols}{2}

\begin{algorithm}[H]
\caption{DiDiCM over Class Prob. (DiDiCM-CP)}
\label{alg:didicm-pd}
\begin{algorithmic}[1]
\Require Instance $\rvy \sim P(\rvy)$. Model $s_\theta$. Noise schedule $\{\sigma_t\}_{t\in[0,1]}$. Step size $\Delta t > 0$.
\State $p(\rc_1|\rvy) := \mathcal{U}(\{1,\dots,K\})$
\State $p_1 \leftarrow p(\rc_1|\rvy)$
\Comment{Initialize noise}
\State $t \leftarrow 1$
\While{$t > 0$}
    \State $j \leftarrow \arg \min p_t$
    \State $s \leftarrow s_\theta(\rvy, j, t)$
    \Comment{Predict the Concrete Score}
    \State $q_t \leftarrow s / \sum_{i=1}^K s_i \in \R^K$
    \State $S \leftarrow q_t \cdot (1 / q_t)^T \in \R^{K \times K}$
    \State $R_t := \sigma_t (\mathbf{1}\mathbf{1}^T - KI) \in \R^{K \times K}$
    \State $\overline{R} \leftarrow S \odot R_t - \operatorname{diag} ( \mathbf{1}^T ( S \odot R_t ) ) \in \R^{K \times K}$
    \State $p_{t-\Delta t} \leftarrow ( I + \overline{R} \Delta t ) \cdot p_t$
    \Comment{Diffuse next posterior}
    \State $t \leftarrow t - \Delta t$
\EndWhile
\State \Return $p_0$
\Comment{Return the estimated posterior}
\end{algorithmic}
\end{algorithm}

\columnbreak

\begin{algorithm}[H]
\caption{DiDiCM over Class Labels (DiDiCM-CL)}
\label{alg:didicm-cl}
\begin{algorithmic}[1]
\Require Instance $\rvy \sim P(\rvy)$. Model $s_\theta$. Noise schedule $\{\sigma_t\}_{t\in[0,1]}$. Step size $\Delta t > 0$.
\State $p(\rc_1|\rvy) := \mathcal{U}(\{1,\dots,K\})$
\State Sample $\rc_1^i \sim p(\rc_1|\rvy)$ for all $1 \leq i \leq N$
\Comment{Init. noise}
\State $t \leftarrow 1$
\For{$i=1$ \textbf{to} $N$}
\Comment{Apply multiple sampling}
    \While{$t > 0$}
        \State $s \leftarrow s_\theta(\rvy, \rc_t^i, t)$
        \Comment{Predict the Concrete Score}
        \State $r \leftarrow \sigma_t \mathbf{1} - \sigma_t K \cdot \mathbf{e}_{\rc_t^i} \in \R^K$
        \State $\overline{r} \leftarrow s \odot r - \mathbf{1}^T (s \odot r) \cdot \mathbf{e}_{\rc_t^i} \in \R^K$
        \State $p_{t-\Delta t | t} \leftarrow \mathbf{e}_{\rc_t^i} + \overline{r} \Delta t \in \R^K$
        \State Sample $\rc_{t - \Delta t} \sim p_{t-\Delta t | t}$
        \Comment{Diffuse next label}
        \State $t \leftarrow t - \Delta t$
    \EndWhile
\EndFor
\State \Return $p_0 \leftarrow \frac{1}{N} \sum_{i=1}^N \mathbf{e}_{\rc_0^i}$
\Comment{Return the est. posterior}
\end{algorithmic}
\end{algorithm}

\end{multicols}
\end{figure*}


In this section, we discuss on methods for simulating the reversal process in order to estimate the posterior distribution $p_\theta(\rc_0|\rvy) \in \mathbb{R}^K$ by exploiting the scoring model $s_\theta$.

A straightforward, yet naive, approach involves predicting $S_t$ using the scoring model $s_\theta$ by constructing the full score matrix
$S_t^\theta := [s_\theta(\rvy, 1, t), \dots, s_\theta(\rvy, K, t)] \in \mathbb{R}^{K \times K}$. 
Leveraging Equation~\ref{eq:reverse ode}, the approximation of the score matrix $S_t^\theta$ can now be used to approximate the reversal transition rate matrix $\overline{R}_t^\theta \approx \overline{R}_t$.
Finally, $\overline{R}_t^\theta$ is used to construct $\overline{Q}_t^\theta := I + \overline{R}_t^\theta \Delta t$ and the reverse diffusion step is applied through:
\begin{align}
\label{eq:pd-ddcm}
    p_\theta(\rc_{t-\Delta t}|\rvy) = \overline{Q}^\theta_t \cdot p_\theta(\rc_t|\rvy)~.
\end{align}

Starting from the uniform distribution $p_\theta(\rc_1|\rvy) := p(\rc_1|\rvy) := \mathcal{U}(\{1,\dots,K\})$,
Equation~\eqref{eq:pd-ddcm} enables the approximation of the target posterior $p_\theta(\rc_0|\rvy) \in \R^K$.  
However, this approach requires $K \frac{1}{\Delta t}$ model iterations per input $\rvy$.
For instance, with $\frac{1}{\Delta t} = 8$ diffusion steps and $K=1000$ class labels, the reverse process would require 8000 model iterations to classify just a single image.

To address this issue, we propose hereafter two alternative approaches for the reverse diffusion process:
(1) DiDiCM over class probabilities (DiDiCM-CP), detailed in Section~\ref{sec:pd-ddcm} and summarized in Algorithm~\ref{alg:didicm-pd};
and (2) DiDiCM over class labels (DiDiCM-CL), detailed in Section~\ref{sec:cl-ddcm} and summarized in Algorithm~\ref{alg:didicm-cl}.
DiDiCM-CP is computationally more efficient, but demands greater memory resources, while DiDiCM-CL is more memory-efficient at the cost of additional computation.


\subsection{DiDiCM over Class Probabilities}
\label{sec:pd-ddcm}

Here, we describe how to apply DiDiCM in the class probabilities space, referred to as \emph{Discrete Diffusion Classification Modeling over Class Probabilities} (DiDiCM-CP). DiDiCM-CP provides a computationally efficient framework for estimating the posterior distribution $p_\theta(\rc_0|\rvy) \in \mathbb{R}^K$ through a single model iteration for each diffusion step in the process.
The complete procedure is summarized in Algorithm~\ref{alg:didicm-pd}. 

A key advantage in classification tasks is that the prior distribution \(P(\rc) \in \R^K\) is fully tractable, allowing explicit iteration over all possible class labels. Consequently, we aim to leverage this property by examining the structure of the score matrix \(S_t\). We observe that \(S_t\) is a rank-one matrix:  
\begin{align}
\label{eq:rank-one scores}
    S_t = q(\rc_t|\rvy) \left( \frac{1}{q(\rc_t|\rvy)} \right)^T ~,
\end{align} 
where \(1/\cdot\) denotes element-wise inversion. This property enables the construction of \(S_t^\theta\) using a single model iteration, thereby reducing the computational complexity of the reverse process to $\frac{1}{\Delta t}$ model iterations.  

Specifically, for any class label \(j \in \{1,\dots,K\}\), the approximation of the forward distribution \(q_\theta(\rc_t|\rvy) \in \R^K\) can be obtained by normalizing $s_\theta(\rvy, j, t)$ as follows,
\begin{align}
\label{eq:normalized scores}
    q_\theta(\rc_t|\rvy) = 
    \frac{s_\theta(\rvy, j, t)}
         {\sum_{i=1}^K s_\theta(\rvy, j, t)_i} \in \R^K ~.
\end{align}

The approximation of the scores matrix $S_t^\theta$ is then obtained by substituting Equation~\eqref{eq:normalized scores} into Equation~\eqref{eq:rank-one scores}, and then is used in the pipeline of Equation~\eqref{eq:pd-ddcm} to achieve the target posterior $p_\theta(\rc_0|\rvy)$.

Note that in theory, the choice of $j$ as an input for the scoring model $s_\theta$ does not affect $q_\theta(\rc_t|\rvy)$ in Equation~\eqref{eq:normalized scores}.
However, we empirically find that setting  
$j := \arg\min p_\theta(\rc_t|\rvy)$, i.e., selecting the class label that \emph{minimizes} the current noisy posterior distribution, achieves the highest performance among other selection strategies.  
An empirical ablation study supporting this choice is provided in Appendix~\ref{app:selection}.

Regarding memory complexity, applying DiDiCM-CP requires maintaining $\mathcal{O}\!\left(K^2\right)$ memory, corresponding to the transition matrix $\overline{Q}^\theta_t \in \mathbb{R}^{K \times K}$ constructed for performing the diffusion steps described in Equation~\eqref{eq:pd-ddcm}. 
In the next subsection, we propose a more memory-efficient approach that works on class label samples drawn from the noisy distributions of the diffusion process. This reduces the dimensionality of the scores to $K$, since the noisy class label is provided at each step.


\subsection{DiDiCM over Class Labels}
\label{sec:cl-ddcm}

We now describe how to apply DiDiCM in the class label space to diffuse noisy class labels, which we refer to as \emph{Diffusion Classification Modeling over Class Labels} (DiDiCM-CL).
DiDiCM-CL enables memory-efficient sampling from the approximate posterior distribution, 
$\rc_0 \sim p_\theta(\rc_0|\rvy)$,
where the posterior itself can be approximated by averaging multiple one-hot sample candidates.
The overall procedure is summarized in Algorithm~\ref{alg:didicm-cl}.


We first describe how to adapt the reverse process of Equation~\eqref{eq: reverse euler} to the class label space by conditioning the distribution of the diffusion step on $\rc_t$ to construct $p(\rc_{t-\Delta t} | \rc_t, \rvy) \in \R^K$.
When conditioned on $\rc_t$, the distribution $p(\rc_t|\rvy)$ reduces to the one-hot encoding of $\rc_t$.  
Consequently, the reverse process in Equation~\eqref{eq: reverse euler} simplifies to:
\begin{align}
    \nonumber
    p(\rc_{t-\Delta t}& = i | \rc_t, \rvy) = \\
    & 
    \begin{cases}
        S_t(i, \rc_t) \, \sigma_t \, \Delta t ~, & \rc_t \neq i \\
        1 - \sum_{\rc \neq \rc_t} S_t(\rc, \rc_t) \, \sigma_t \, \Delta t ~, & \rc_t = i
    \end{cases} ~.
\end{align}
Since our scoring model approximates the columns of $S_t$, a single model evaluation can be used to approximate the above conditioned distribution as follows:
\begin{align}
\label{eq:cl-ddcm}
    \nonumber
    p_\theta(&\rc_{t-\Delta t} = i | \rc_t, \rvy) = \\
    & 
    \begin{cases}
        s_\theta(\rvy, \rc_t, t)_i \, \sigma_t \, \Delta t ~, & \rc_t \neq i \\
        1 - \sum_{i\neq\rc_t} s_\theta(\rvy, \rc_t, t)_i \, \sigma_t \, \Delta t ~, & \rc_t = i
    \end{cases} ~.
\end{align}

Starting from a random class label $\rc_1 \sim p(\rc_1 | \rvy) := \mathcal{U}(\{1,\dots,K\})$, one can iteratively diffuse noisy class labels to generate a label candidate $\rc_0 \sim p_\theta(\rc_0|\rvy)$ using Equation~\eqref{eq:cl-ddcm}.
However, a single sample $\rc_0$ may not represent the class label of $\rvy$. 
To obtain an estimate for the posterior distribution $p_\theta(\rc_0|\rvy)$, we propose to apply multiple reverse processes on the same input $\rvy$ to generate a set of samples $\{\rc_0^i\}_{i=1}^N$ s.t. $\rc_0^i \sim p_\theta(\rc_0|\rvy)$.
Then $p_\theta(\rc_0|\rvy)$ can be estimated by the Monte-Carlo approximation:
\begin{equation}
    p_\theta(\rc_0|\rvy)
    = \E_{\rc_0 \sim p_\theta(\rc_0|\rvy)} \big[ \mathbf{e}_{\rc_0} \big]
    \approx \frac{1}{N} \sum_{i=1}^N \mathbf{e}_{\rc_0^i} ~,
\end{equation}
where $\mathbf{e}_{\rc_0^i}$ is the one-hot encoding of $\rc_0^i$.

Note that applying DiDiCM-CL requires maintaining $\mathcal{O}\!(K+N)$ memory, corresponding to the transition vector that is used in Equation~\eqref{eq:cl-ddcm} and to the maintenance of the class label samples.
Regarding computational complexity, the sampling process requires $\frac{1}{\Delta t}$ model iterations for each input $\rvy$.
When we aim to estimate $p_\theta(\rc_0|\rvy)$, this complexity increases to $N\frac{1}{\Delta t}$.
Nevertheless, we show in Section~\ref{sec:exp} that setting 32 NFEs are sufficient to achieve near-upper-bound image classification performance on the ImageNet dataset.

\begin{figure}[t]
    \centering
    \includegraphics[width=0.85\columnwidth]{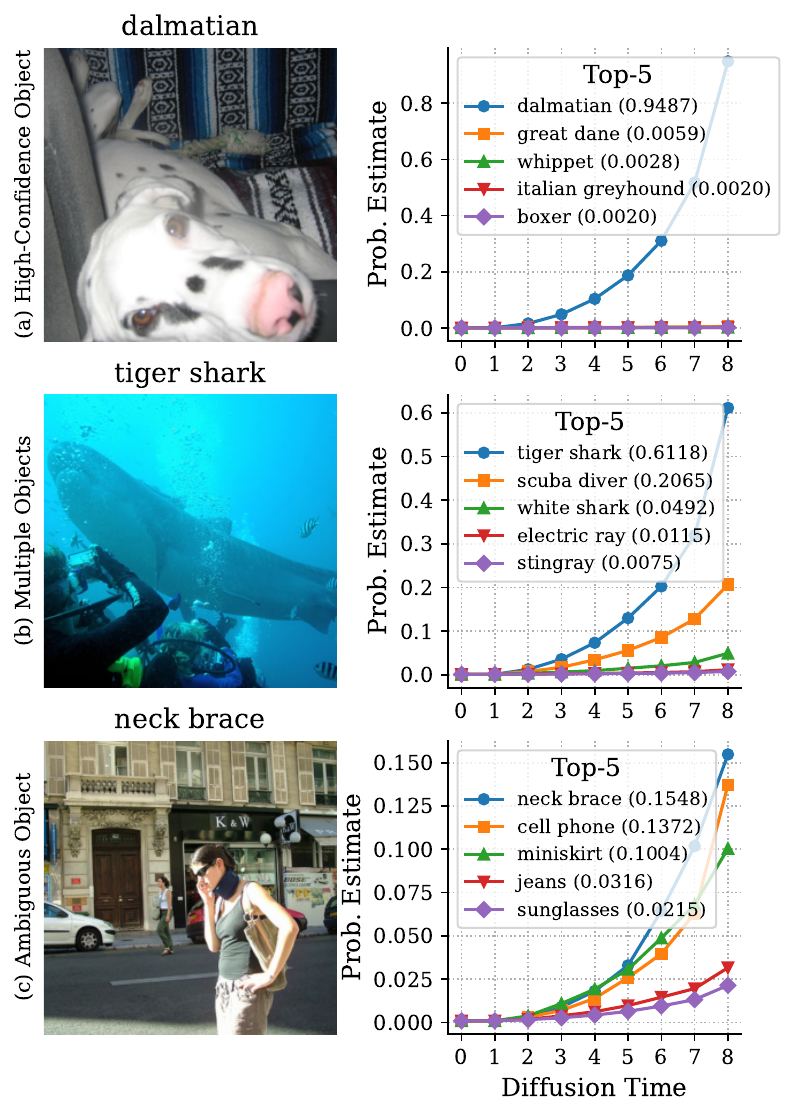}
    \caption{An illustration of our DiDiCM-CP showing the evolution of the top-5 class label probabilities over time for three images, demonstrating different classification challenges.}
    \label{fig:viz}
\end{figure}

\section{Empirical Study}
\label{sec:exp}

In this section, we empirically demonstrate the advantages of using DiDiCM for image classification on the challenging ImageNet-1k (ILSVRC-2012) dataset \cite{russakovsky2015imagenet}.
Figure~\ref{fig:viz} shows the evolution of the top-5 class labels probabilities over time for three images, each demonstrate a different type of classification challenge.
Additional visualizations of misclassified images are provided in Appendix~\ref{app:bad}.

We compare DiDiCM against the standard classification framework using state-of-the-art training techniques. Our primary results, presented in Table~\ref{tab:main_results}, indicate that the discrete diffusion approach matches or outperforms standard classification models under low-uncertainty conditions and significantly surpasses their performance when uncertainty is included, due to its stochastic nature. These findings suggest that our discrete diffusion method provides a clear advantage over standard classifiers in image classification.

Implementing our core scoring model, denoted as $s_\theta$, requires additional components to handle both the noisy class label $\rc_t$ and its corresponding noise level $t$, alongside the modules responsible for image processing. Consequently, a direct comparison with standard classification architectures is not straightforward. To enable a fair evaluation, each conventional architecture must be adapted to incorporate the conditioning require by DiDiCM, while preserving its core image processing components.

To this end, we propose a novel architecture, which enables a direct comparison with the performance of the well-established ResNet as reported in the literature.



\begin{figure}[t]
    \centering
    \includegraphics[width=0.95\columnwidth]{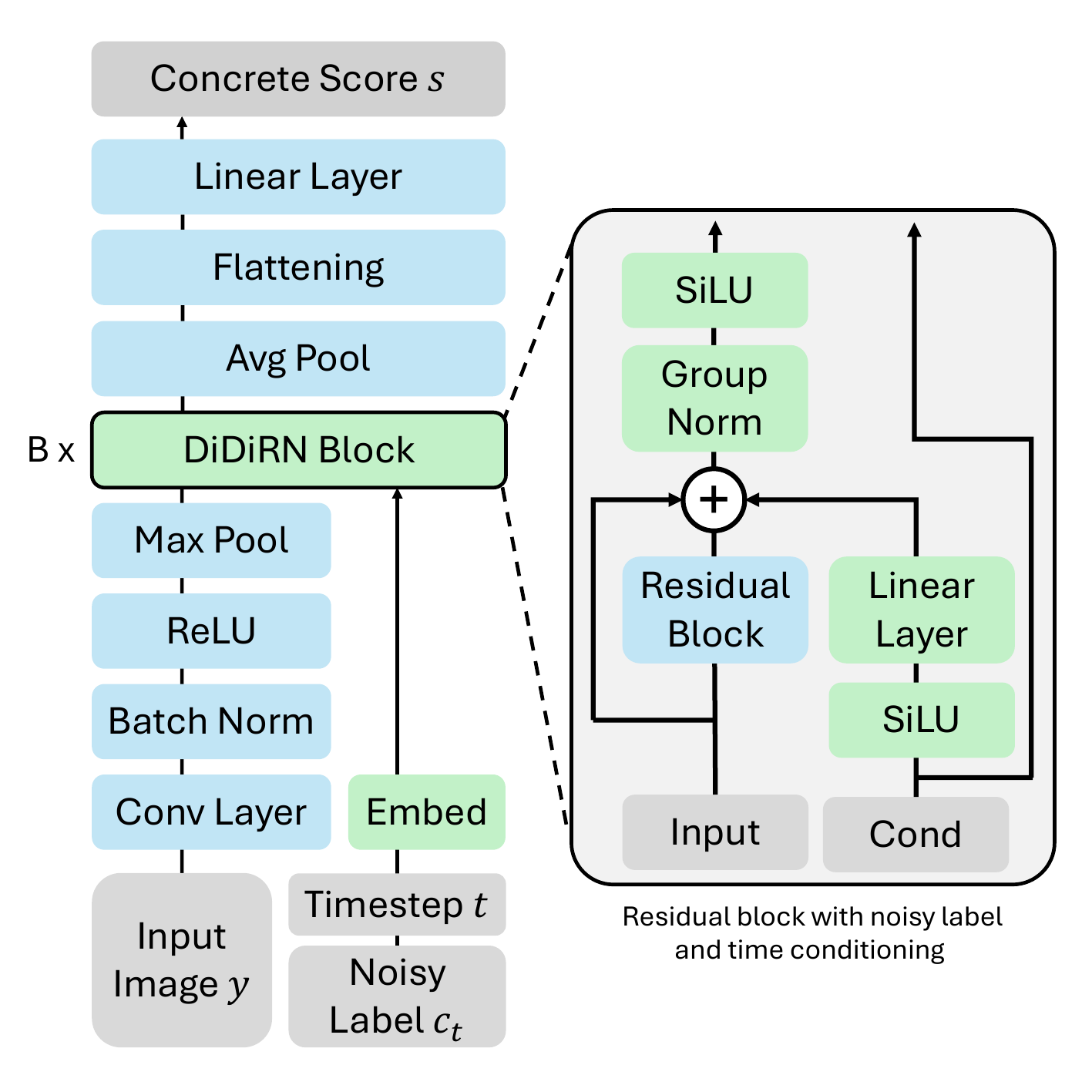}
    \caption{\textbf{The Discrete Diffusion Residual Network (DiDiRN) architecture.}
    DiDiRN preserves the core image-processing components of ResNet while adding conditioning modules to support the diffusion process of DiDiCM. Original ResNet modules are shown in \textcolor{blue!60}{blue}, and the newly introduced components in \textcolor{Green}{green}.}
    \label{fig:didirn}
\end{figure}

\subsection{The Discrete Diffusion Residual Network}
\label{sec:didirn}

Here, we introduce the \emph{Discrete Diffusion Residual Network} (DiDiRN), a variant of the ResNet architecture that incorporates conditioning mechanisms inspired by Guided Diffusion~\citep{dhariwal2021diffusion} to process the noisy class label and its associated noise level. Intuitively, DiDiRN is designed to enable lightweight conditioning on noisy class labels within the feature extraction layers of ResNet, while preserving its original convolutional structure used for image processing. The architectural modifications are illustrated in Figure~\ref{fig:didirn}.

In particular, the ResNet architecture consists of three stages:
(1) converting the input image into a set of low-level base features through initial convolution and pooling;
(2) performing deep feature extraction using a sequence of B x residual blocks that progressively refine the representation, supported by residual connections;
and (3) applying classification layers that aggregate the extracted features using global average pooling, followed by a fully connected layer to produce the final output predictions.

Inspired by the classifier-free guidance modules of Guided Diffusion \citep{dhariwal2021diffusion},
we modify the ResNet architecture described above to construct the DiDiRN building blocks.
Specifically, in Stage (1), we introduce initial label and time embedding layers, each producing embedding vectors of identical dimensionality specified by a hyperparameter. These embeddings are then combined into a single summarized embedding vector.
In Stage (2), corresponding to the feature extraction process, we adjust each residual block by first transforming the summarized embedding vector using a SiLU activation followed by a linear layer, and then adding the resulting vector to the skip connection.
The summarized vector therefore encodes a combination of the residual, the block input, and the conditioning.
Finally, we normalize the outcomes using a GroupNorm layer and apply another SiLU activation.


\begin{table*}[htb]
\centering
\small
\begin{tabular}{c|c|c||ccc|ccc|ccc}
\toprule
\multirow{2}{*}{\textbf{Model}} & \multirow{2}{*}{\textbf{Aug. Recipe}} & \multirow{2}{*}{\textbf{Metric}} & \multicolumn{3}{c|}{\textbf{Training Ratio - 25\%}} & \multicolumn{3}{c|}{\textbf{Training Ratio - 50\%}} & \multicolumn{3}{c}{\textbf{Training Ratio - 100\%}} \\
 &  &  & 56 & 112 & 224 & 56 & 112 & 224 & 56 & 112 & 224 \\
\midrule
\multirow{4}{*}{\centering\shortstack{Standard\\Classifiers\\\\(ResNet-50)}}
 & \multirow{2}{*}{Weak Aug} 
 & \cellcolor{blue!10}Top-1 & \cellcolor{blue!10}44.86 & \cellcolor{blue!10}57.82 & \cellcolor{blue!10}64.48 & \cellcolor{blue!10}54.33 & \cellcolor{blue!10}66.19 & \cellcolor{blue!10}71.41 & \cellcolor{blue!10}61.75 & \cellcolor{blue!10}72.55 & \cellcolor{blue!10}76.57 \\
 &  & \cellcolor{blue!10}Top-5 & \cellcolor{blue!10}66.46 & \cellcolor{blue!10}77.72 & \cellcolor{blue!10}82.80 & \cellcolor{blue!10}75.68 & \cellcolor{blue!10}84.78 & \cellcolor{blue!10}88.41 & \cellcolor{blue!10}82.58 & \cellcolor{blue!10}89.84 & \cellcolor{blue!10}92.42 \\
 
 & \multirow{2}{*}{Strong Aug} 
 & \cellcolor{green!10}Top-1 & \cellcolor{green!10}53.77 & \cellcolor{green!10}66.51 & \cellcolor{green!10}71.85 & \cellcolor{green!10}60.51 & \cellcolor{green!10}72.31 & \cellcolor{green!10}\textbf{77.05} & \cellcolor{green!10}64.71 & \cellcolor{green!10}75.96 & \cellcolor{green!10}\textbf{80.42} \\
 &  & \cellcolor{green!10}Top-5 & \cellcolor{green!10}76.03 & \cellcolor{green!10}85.41 & \cellcolor{green!10}88.72 & \cellcolor{green!10}81.74 & \cellcolor{green!10}89.81 & \cellcolor{green!10}92.64 & \cellcolor{green!10}85.32 & \cellcolor{green!10}92.30 & \cellcolor{green!10}94.60 \\
\midrule
\multirow{4}{*}{\centering\shortstack{\textbf{Ours:} \\\textbf{DiDiCM-CP}\\$1/\Delta t = 8$\\\\(DiDiRN-50)}}
 & \multirow{2}{*}{Weak Aug} 
 & \cellcolor{blue!10}Top-1 & \cellcolor{blue!10}\textcolor{Green}{57.92} & \cellcolor{blue!10}65.48 & \cellcolor{blue!10}68.08 & \cellcolor{blue!10}\textcolor{Green}{64.85} & \cellcolor{blue!10}71.69 & \cellcolor{blue!10}73.73 & \cellcolor{blue!10}\textcolor{Green}{69.87} & \cellcolor{blue!10}\textcolor{Green}{76.30} & \cellcolor{blue!10}77.68 \\
 &  & \cellcolor{blue!10}Top-5 & \cellcolor{blue!10}\textcolor{Green}{80.26} & \cellcolor{blue!10}\textcolor{Green}{85.75} & \cellcolor{blue!10}87.38 & \cellcolor{blue!10}\textcolor{Green}{85.17} & \cellcolor{blue!10}89.70 & \cellcolor{blue!10}91.14 & \cellcolor{blue!10}\textcolor{Green}{88.95} & \cellcolor{blue!10}\textcolor{Green}{92.50} & \cellcolor{blue!10}93.37 \\
 
 & \multirow{2}{*}{Strong Aug} 
 & \cellcolor{green!10}Top-1 & \cellcolor{green!10}\textbf{59.05} & \cellcolor{green!10}\textbf{68.83} & \cellcolor{green!10}\textbf{72.27} & \cellcolor{green!10}\textbf{65.75} & \cellcolor{green!10}\textbf{73.67} & \cellcolor{green!10}\textbf{77.01} & \cellcolor{green!10}\textbf{70.80} & \cellcolor{green!10}\textbf{77.89} & \cellcolor{green!10}\textbf{80.40} \\
 &  & \cellcolor{green!10}Top-5 & \cellcolor{green!10}\textbf{81.93} & \cellcolor{green!10}\textbf{88.76} & \cellcolor{green!10}\textbf{91.00} & \cellcolor{green!10}\textbf{86.72} & \cellcolor{green!10}\textbf{91.63} & \cellcolor{green!10}\textbf{93.62} & \cellcolor{green!10}\textbf{89.69} & \cellcolor{green!10}\textbf{93.75} & \cellcolor{green!10}\textbf{95.29} \\
\bottomrule
\end{tabular}
\caption{\textbf{ImageNet Top-1 and Top-5 Accuracy:} DiDiCM (8 steps) vs. standard classifiers under varying uncertainty. Weak Aug uses standard PyTorch augmentations \cite{paszke2019pytorch}, Strong Aug follows the state-of-the-art ResNet recipe \cite{wightman2021resnet} (see Appendix~\ref{app:training} for more details). Best top-1 and top-5 per column are bolded; Weak Aug models outperforming Strong Aug classifiers are highlighted in \textcolor{Green}{green}.}
\label{tab:main_results}
\end{table*}

\begin{figure*}[!t]
    \centering
    \begin{subfigure}[t]{0.58\textwidth}
        \centering
        \includegraphics[width=\textwidth]{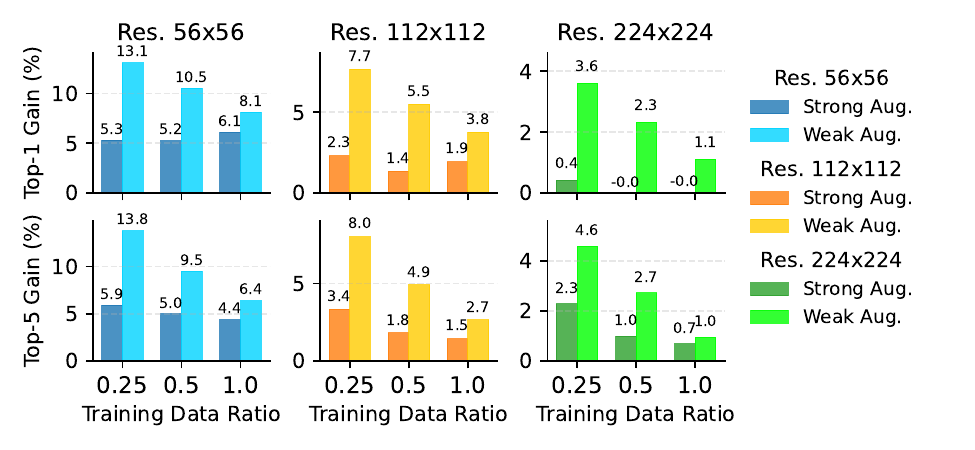}
        \caption{\textbf{ImageNet Top-1 and Top-5 Accuracy Gains}}
        \label{fig:gain}
    \end{subfigure}%
    \hfill
    \begin{subfigure}[t]{0.42\textwidth}
        \centering
        \includegraphics[width=\textwidth]{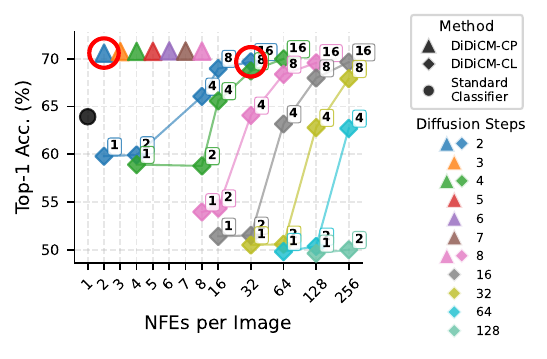}
        \caption{\textbf{Quality-Efficiency Analysis}}
        \label{fig:nfe}
    \end{subfigure}
    \caption{(a) DiDiCM (8 steps) vs. standard classifiers under varying uncertainty (see Appendix~\ref{app:training} for augmentation policy). (b) NFEs vs. top-1 accuracy for DiDiCM-CP, DiDiCM-CL, and the standard classifier at resolution 56 using the full training set. Numbers indicate the sample count used. Red markers denote the best-performing DiDiCM-CP and DiDiCM-CL results.}
    \label{fig:gain_and_nfe}
    \vspace*{-0.1cm}
\end{figure*}

\subsection{DiDiCM vs. Standard Classifiers on ImageNet}

We now discuss our experimental results in Table~\ref{tab:main_results}, comparing DiDiCM with standard classifiers. Accuracy gains are shown in Figure~\ref{fig:gain}, and the quality-efficiency tradeoff of DiDiCM is illustrated in Figure~\ref{fig:nfe}.

\textbf{Experimental Setup}. ~
All experiments in this section are conducted using the ResNet-50 architecture. For DiDiCM, we construct the corresponding DiDiRN-50 model. In Appendix~\ref{app:all_resnet}, we extend the analysis to all standard ResNet model sizes.
We organize the experiments into four groups. For each group, we train both standard classification models and DiDiCM scoring models using two variants of the A1 training recipe from ResNet-SB~\cite{wightman2021resnet}, which represents the current state-of-the-art for ResNet-50.
The first variant follows ResNet-SB but replaces its data augmentation pipeline with the standard PyTorch implementation~\cite{paszke2019pytorch}, denoted as \emph{Weak Aug}. The second variant retains the full recipe, denoted as \emph{Strong Aug}.
Each experimental group comprises nine uncertainty settings, defined by varying image corruption levels across three input resolutions (224, 112, and 56) and by randomly sampling class-balanced subsets of the ImageNet training data at three ratios (0.25, 0.5, and 1.0). More training details are provided in Appendix~\ref{app:training}.


\textbf{Strong Aug vs. Strong Aug}. ~
We first compare the performance of DiDiCM and standard classifiers using the state-of-the-art ResNet-SB training recipe \cite{wightman2021resnet}. DiDiCM consistently outperforms standard classifiers in top-5 accuracy, with gains increasing under higher uncertainty (see Figure~\ref{fig:main_figure}). For top-1 accuracy, DiDiCM also surpasses standard classifiers in uncertainty-challenging settings, while achieving comparable performance under low-uncertainty conditions (e.g., resolution 224, full or half training data). Our best DiDiCM result, obtained under no-uncertainty conditions, achieves 80.4\% top-1 accuracy, 
which is approximately the highest reported value for ResNet-50 \cite{wightman2021resnet}.

\textbf{Weak Aug vs. Weak Aug}. ~
Next, we compare DiDiCM with standard classifiers using the standard image augmentation pipeline (see Appendix~\ref{app:training}).
DiDiCM consistently achieves higher top-1 and top-5 accuracy across all uncertainty conditions, with performance gains increasing as uncertainty grows, reaching 13.1\% (top-1) and 13.8\% (top-5) accuracy gains at the highest uncertainty level (resolution 56, training ratio 0.25).

\textbf{Weak Aug vs. Strong Aug}. ~
We also compare DiDiCM under the Weak Aug policy with standard classifiers using the state-of-the-art training recipe, referred to as Strong Aug. DiDiCM outperforms standard classifiers under challenging uncertainty conditions, suggesting it is easier to train and can surpass state-of-the-art methods with a simpler training policy.

\textbf{Quality vs. Efficiency}. ~
A notable limitation of DiDiCM compared to standard classifiers is its dependence on sequential model evaluations for label prediction.
To examine this, we analyze the relationship between NFEs and top-1 accuracy for DiDiCM-CP and DiDiCM-CL. As shown in Figure~\ref{fig:nfe}, under a 56 resolution with full training data, DiDiCM-CP reaches near-upper-bound top-1 accuracy with only two diffusion steps, whereas DiDiCM-CL achieves its best performance at 32 NFEs using two diffusion steps and $N=16$ class labels. Table~\ref{tab:nfe} in Appendix~\ref{app:nfe} summarizes results across all uncertainty conditions with these parameters fixed. Under these efficient settings, DiDiCM-CP outperforms standard classifiers with similar accuracy-efficiency trends.

\section{Concluding Remarks}

This paper introduced \emph{Discrete Diffusion Classification Modeling} (DiDiCM), an efficient and high-performing diffusion-based framework for classification. DiDiCM refines class probabilities through a tractable score-based learning process of the Concrete Score.
To balance computational and memory efficiency, we propose two simulation variants: DiDiCM-CP, which operates over class probabilities and offers higher computational efficiency at the cost of increased memory usage, and DiDiCM-CL, which operates over class labels and is thus more memory-efficient.
Moreover, we presented the \emph{Discrete Diffusion Residual Network} (DiDiRN), a DiDiCM-compatible architecture enabling direct comparison with well-established ResNet benchmarks. Experimental results demonstrate notable gains in top-1 and top-5 accuracy over standard classifiers, even with few diffusion steps.
Future research directions may include exploring more advanced architectural designs to further enhance performance and developing diffusion distillation strategies to improve inference efficiency even further.

\newpage

{
    \small
    \bibliographystyle{ieeenat_fullname}
    \bibliography{main}
}

\clearpage
\setcounter{page}{1}
\maketitlesupplementary

\newtheorem{theorem}{Theorem}

\appendix

\section{Background on Transition Matrices}
\label{app:rate_matrices}

In this section, we present the theoretical background and intuition underlying the transition rate matrices used in our diffusion processes. These matrices form the core of our framework, governing the temporal evolution of probabilities through a continuous-time Markov process.

\subsection{Theoretical Analysis of Transition Rates}
\label{app:theoretical_analysis}

We begin by generalizing our diffusion framework to an arbitrary target probability vector. Let $q_0 \in \mathbb{R}^K$ denote the target distribution, and let $q_t \in \mathbb{R}^K$ represent its time-dependent evolution for $t \in [0,1]$. Assume that $q_t$ evolves  according to the linear ordinary differential equation,
\begin{align}
\label{eq:forward_appendix}
    \frac{dq_t}{dt} = R_t q_t ~,
\end{align}
where $R_t \in \mathbb{R}^{K \times K}$ is the transition rate matrix.

Following \citet{anderson2012continuous} and \citet{campbell2022continuous}, the transition rate matrix is defined by,
\begin{align}
    R_t(i,j) := \lim_{\Delta t \rightarrow 0} \frac{q_{t|t-\Delta t}(j|i) - \delta_{i,j}}{\Delta t} ~,
\end{align}
with $q_{t|t-\Delta t}(j|i)$ representing the infinitesimal probability of transitioning from state $i$ at time $t - \Delta t$ to state $j$ at time $t$.
For $i \neq j$, $R_t(i,j)$ thus quantifies the instantaneous rate of transition from state $i$ to state $j$; higher values correspond to faster expected transitions.

By construction, $R_t(i,j) \geq 0$ for $i \neq j$ and $R_t(j,j) = -\sum_{i \ne j} R_t(i,j) \leq 0$, ensuring that each column of $R_t$ sums to zero. Therefore, $R_t$ preserves the total probability mass when applied to $q_t$ and defines a valid Markov process generator.

Suppose $R_t$ takes the simple structured form $R_t = \sigma_t R$, with $R\in\R^{K \times K}$ being a time-invariant transition rate matrix and $\sigma_t$ is a strictly decreasing non-negative scaling function satisfying $\sigma_0 > 0$ and $\sigma_1 = 0$.
Under this assumption, Equation~\eqref{eq:forward_appendix} admits a closed-form solution.
Below is a proof.

\begin{theorem}[Closed-Form Solution for the Discrete Markovian Process]
\label{the:total_forward}
Let $q_t \in \mathbb{R}^K$ satisfy Equation~\eqref{eq:forward_appendix} and let $R = U \Lambda U^{-1}$ be the eigendecomposition of $R$, where $U \in \mathbb{R}^{K \times K}$ denotes the matrix of eigenvectors 
and $\Lambda \in \mathbb{R}^{K \times K}$ is the diagonal matrix of eigenvalues. Define
$\overline{\sigma}_t := \int_0^t \sigma_s \, ds$.
Then,
\begin{align}
\label{eq:total_forward_appendix}
    q_t = U \exp\left( \overline{\sigma}_t \Lambda \right) U^{-1} q_0 ~.
\end{align}
\end{theorem}

\begin{proof}
Assume \( q_t = \exp(\overline{\sigma}_t R) q_0 \).
We first show that this form satisfies the given differential equation.

Differentiating with respect to \( t \),
\begin{align}
\frac{dq_t}{dt} 
    &= \frac{d}{dt} \left( \exp(\overline{\sigma}_t R) \right) q_0 \\
    &= \exp(\overline{\sigma}_t R) \frac{d}{dt}(\overline{\sigma}_t R) q_0 \\
    &= \exp(\overline{\sigma}_t R) \sigma_t R q_0 = \exp(\overline{\sigma}_t R) R_t q_0 ~.
\end{align}

Since the matrices $\exp(\overline{\sigma}_t R)$ and $R_t$ are polynomials in $R$, they are diagonalizable and therefore commute.
It then follows that,
\begin{align}
    \frac{dq_t}{dt} = \exp(\overline{\sigma}_t R) R_t q_0 = R_t \exp(\overline{\sigma}_t R) q_0 = R_t q_t ~,
\end{align}
thus confirming the assumption made,
\begin{align}
    q_t = \exp(\overline{\sigma}_t R) q_0 ~.
\end{align}

Additionally, Appendix E of \citet{campbell2022continuous} establishes that,
\begin{align}
    \exp(\overline{\sigma}_t R) 
    &= \sum_{k=0}^{\infty} \frac{1}{k!} \left( \overline{\sigma}_t R \right)^k \\
    &= \sum_{k=0}^{\infty} \frac{1}{k!} \left( U \Lambda U^{-1} \overline{\sigma}_t \right)^k \\
    &= \sum_{k=0}^{\infty} \frac{1}{k!} U \left( \Lambda \overline{\sigma}_t \right)^k U^{-1} \\
    &= U \left\{ \sum_{k=0}^{\infty} \frac{1}{k!} \left( \Lambda \overline{\sigma}_t \right)^k \right\} U^{-1} \\
    &= U \exp\left( \overline{\sigma}_t \Lambda \right) U^{-1} ~.
\end{align}
Therefore, the stated solution follows.
\end{proof}

When $R_t$ does not adopt the form $\sigma_t R$, for instance, in the case of the reversal rate matrix $\overline{R}_t$ introduced in Section~\ref{sec:reverse}, an approximate solution can be obtained via Euler discretization with sufficiently small time increment $\Delta t > 0$.
In this approach, the infinitesimal transition matrix is defined as
$Q_t := I + R_t \Delta t$.
The probability distribution is then updated according to $q_{t+\Delta t} \approx Q_t q_t$,
where each entry satisfies $Q_t(i,j) \geq 0$, and the columns of $Q_t$ sum to one, thereby ensuring conservation of total probability mass when applied to $q_t$.
Below is a proof for this solution.

\begin{theorem}[Approximated Solution for the Discrete Markovian Process]
\label{the:matrix-euler}
Let $q_t \in \mathbb{R}^K$ satisfy Equation~\eqref{eq:forward_appendix}.
For sufficiently small time increment $\Delta t > 0$, the first-order Euler approximation yields
\begin{align}
    q_{t + \Delta t} = Q_t q_t + \mathcal{O}(\Delta t^2) ~,
\end{align}
where $\mathcal{O}(\Delta t^2)$ denotes a remainder term that vanishes quadratically as $\Delta t \to 0$.
\end{theorem}

\begin{proof}
By the definition of the time derivative, we have,
\[
    \frac{dq_t}{dt} = \lim_{\Delta t \to 0} \frac{q_{t+\Delta t} - q_t}{\Delta t} = R_t q_t.
\]
Consequently, for sufficiently small \( \Delta t > 0 \),
\begin{align}
    \frac{q_{t+\Delta t} - q_t}{\Delta t} &= R_t q_t + \mathcal{O}(\Delta t), \\
    q_{t+\Delta t} - q_t &= R_t q_t \, \Delta t + \mathcal{O}(\Delta t^2), \\
    q_{t+\Delta t} &= \bigl( I + R_t \Delta t \bigr) q_t + \mathcal{O}(\Delta t^2) \\
    q_{t+\Delta t} &= Q_t q_t + \mathcal{O}(\Delta t^2) ~.
\end{align}
\end{proof}

\subsection{Intuition of the Forward Transition Rates}
\label{app:forward}

The forward process introduced in Section~\ref{sec:forward} relies on a specific structure of the transition rate matrix $R_t$, defined as $R_t = \sigma_t R$, where $R := \mathbf{1}\mathbf{1}^T - K I \in \mathbb{R}^{K \times K}$.
This form, referred to as the uniform transition rate matrix \cite{anderson2012continuous,campbell2022continuous}, and can be written explicitly as
\begin{align}
    R_t = \sigma_t \begin{bmatrix}
        1-K & 1 & \cdots & 1 \\
        1 & 1-K & \cdots & 1 \\
        \vdots & \vdots & \ddots & \vdots \\
        1 & 1 & \cdots & 1-K
    \end{bmatrix} ~.
\end{align}
Intuitively, at each time $t$, this matrix transitions the probability vector to the uniform distribution, controlled by the scale factor $\sigma_t$.

For illustration, consider $K=3$, $\sigma_t=0.2$, and $q_t = [0.8, 0.1, 0.1]^T$. Then,
\begin{align}
    \nonumber
    \frac{dq_t}{dt} &= R_t q_t \\
    &= \begin{bmatrix}
        -0.4 & 0.2 & 0.2 \\
        0.2 & -0.4 & 0.2 \\
        0.2 & 0.2 & -0.4
    \end{bmatrix}
    \begin{bmatrix}
        0.8 \\ 0.1 \\ 0.1
    \end{bmatrix}
    = \begin{bmatrix}
        -0.28 \\ 0.14 \\ 0.14
    \end{bmatrix} ~.
\end{align}
Now, increasing $\sigma_t$ to $2.5$ yields,
\begin{align}
    \nonumber
    \frac{dq_t}{dt} &= R_t q_t \\
    &= \begin{bmatrix}
        -5 & 2.5 & 2.5 \\
        2.5 & -5 & 2.5 \\
        2.5 & 2.5 & -5
    \end{bmatrix}
    \begin{bmatrix}
        0.8 \\ 0.1 \\ 0.1
    \end{bmatrix}
    = \begin{bmatrix}
        -3.5 \\ 1.75 \\ 1.75
    \end{bmatrix} ~.
\end{align}

From these computations, we observe that larger values of $\sigma_t$ lead to faster transitions, causing the probability mass to spread more quickly across the discrete states.
The uniform structure of $R_t$ guarantees equal transition potential among states, meaning that states with higher source probabilities will experience proportionally higher transition rates than others.

\subsection{Intuition of the Reverse Transition Matrices}
\label{app:reverse}

Recalling the reverse process introduced in Equation~\eqref{eq:reverse ode},
\begin{align}
    \nonumber
    \frac{d p_{1-t}}{d t}
    &= \overline{R}_{1-t} \, p_{1-t}, \\
    \text{s.t} \quad
    \overline{R}_t
    &:= S_t \odot R_t - \operatorname{diag} \!\left( \mathbf{1}^T (S_t \odot R_t) \right) ~,
\end{align}
where, for simplicity, we denote $p(\rc_t|\rvy)$ as $p_t$.
For a sufficiently small time step $\Delta t > 0$, applying the Euler approximation from Theorem~\ref{the:matrix-euler} yields,
\begin{align}
    p_{t-\Delta t} \approx \overline{Q}_t p_t ~,
\end{align}
where $\overline{Q}_t := I + \overline{R}_t \Delta t$ represents the infinitesimal transition matrix.

Consider a simple example where $\overline{R}_t = \sigma_t R$ is a uniform transition rate matrix with parameters $K = 3$, $\sigma_t = 0.2$, $\Delta t = 0.1$, and $p_t = [0.8, 0.1, 0.1]^T$.
Then,
\begin{align}
    p_{t - \Delta t}
    &\approx (I + 0.1 \, \overline{R}_t) \, p_t \nonumber \\
    &=
    \begin{bmatrix}
        0.96 & 0.02 & 0.02 \\
        0.02 & 0.96 & 0.02 \\
        0.02 & 0.02 & 0.96
    \end{bmatrix}
    \begin{bmatrix}
        0.8 \\ 0.1 \\ 0.1
    \end{bmatrix}
    =
    \begin{bmatrix}
        0.774 \\ 0.113 \\ 0.113
    \end{bmatrix} ~.
\end{align}

If we increase $\sigma_t$ to 2.5, we obtain
\begin{align}
    p_{t - \Delta t}
    &\approx (I + 0.1 \, \overline{R}_t) \, p_t \nonumber \\
    &=
    \begin{bmatrix}
        0.5 & 0.25 & 0.25 \\
        0.25 & 0.5 & 0.25 \\
        0.25 & 0.25 & 0.5
    \end{bmatrix}
    \begin{bmatrix}
        0.8 \\ 0.1 \\ 0.1
    \end{bmatrix}
    =
    \begin{bmatrix}
        0.5 \\ 0.25 \\ 0.25
    \end{bmatrix}.
\end{align}

A larger $\sigma_t$ therefore induces stronger state transitions, driving the distribution more rapidly toward uniformity in a single step.
In general, the scaling factor $\sigma_t$ governs the sharpness of the reverse diffusion: small values yield smoother, more stable refinements of the probabilities, while large values accelerate convergence but risk oversmoothing finer distinctions in $p_t$.

\begin{table*}[htb]
\centering
\small
\begin{tabular}{c|c|c||ccc|ccc|ccc}
\toprule
\multirow{2}{*}{\textbf{Method}} & \multirow{2}{*}{\shortstack{\textbf{Label Selection}\\\textbf{ Method}}} & \multirow{2}{*}{\textbf{Metric}} & \multicolumn{3}{c|}{\textbf{Training Ratio - 25\%}} & \multicolumn{3}{c|}{\textbf{Training Ratio - 50\%}} & \multicolumn{3}{c}{\textbf{Training Ratio - 100\%}} \\
 &  &  & 56 & 112 & 224 & 56 & 112 & 224 & 56 & 112 & 224 \\
\midrule
\multirow{6}{*}{\shortstack{\textbf{Ours:}\\\textbf{DiDiCM-CP}\\$1/\Delta t = 8$}} 
 & \multirow{2}{*}{ArgMax} & Top-1 
& 56.58 & 63.35 & 69.17 & 65.91 & 72.12 & 77.00 & 69.41 & 76.20 & 79.46 \\
 && Top-5 & 80.99 & 85.89 & 88.61 & 87.48 & 90.75 & 92.85 & 89.89 & 92.88 & 94.72 \\
 & \multirow{2}{*}{Sampling} & \cellcolor{blue!10}Top-1 & \cellcolor{blue!10}58.89 & \cellcolor{blue!10}65.33 & \cellcolor{blue!10}69.78 & \cellcolor{blue!10}\textbf{68.34} & \cellcolor{blue!10}73.38 & \cellcolor{blue!10}\textbf{77.62} & \cellcolor{blue!10}\textbf{71.86} & \cellcolor{blue!10}76.84 & \cellcolor{blue!10}80.27 \\
 &  & \cellcolor{blue!10}Top-5 & \cellcolor{blue!10}\textbf{81.98} & \cellcolor{blue!10}86.74 & \cellcolor{blue!10}89.24 & \cellcolor{blue!10}\textbf{88.80} & \cellcolor{blue!10}91.51 & \cellcolor{blue!10}\textbf{93.79} & \cellcolor{blue!10}\textbf{90.97} & \cellcolor{blue!10}93.53 & \cellcolor{blue!10}95.20 \\
 
 & \multirow{2}{*}{ArgMin} 
 & \cellcolor{green!10}Top-1 & \cellcolor{green!10}\textbf{59.05} & \cellcolor{green!10}\textbf{68.83} & \cellcolor{green!10}\textbf{72.27} & \cellcolor{green!10}65.75 & \cellcolor{green!10}\textbf{73.67} & \cellcolor{green!10}77.01 & \cellcolor{green!10}70.80 & \cellcolor{green!10}\textbf{77.89} & \cellcolor{green!10}\textbf{80.40} \\
 &  & \cellcolor{green!10}Top-5 & \cellcolor{green!10}\textbf{81.93} & \cellcolor{green!10}\textbf{88.76} & \cellcolor{green!10}\textbf{91.00} & \cellcolor{green!10}86.72 & \cellcolor{green!10}\textbf{91.63} & \cellcolor{green!10}93.62 & \cellcolor{green!10}89.69 & \cellcolor{green!10}\textbf{93.75} & \cellcolor{green!10}\textbf{95.29} \\
\bottomrule
\end{tabular}
\caption{Ablation study of label selection strategies in DiDiCM-CP under varying uncertainty conditions.}
\label{tab:selection}
\end{table*}

\section{Ablation Study of Label Selection Strategies in DiDiCM-CP}
\label{app:selection}

In Section~\ref{sec:pd-ddcm}, we introduced DiDiCM-CP, our diffusion-based method over class probabilities, summarized in Algorithm~\ref{alg:didicm-pd}. 
This method takes advantage of the tractable prior distribution \( P(\rc) \) in classification settings, enabling the computation of the complete score matrix \( S_t^\theta \) within a single model iteration. 
This is achieved by normalizing the Concrete Score prediction, defined as 
\( s_\theta(\rvy, j, t) \approx q(\rc_t | \rvy) / q(\rc_t = j | \rvy) \), 
where \( j \in \{1, \dots, K\} \) represents the noisy class label corresponding to noise level \( t \), and \( \rvy \) denotes the input to be classified.

By construction, any label \( j \) can be used in this process. However, we find empirically that the choice of \( j \) can influence model performance. 
In this section, we therefore examine several strategies for selecting \( j \).

We evaluate three label selection strategies:
\begin{enumerate}
    \item \textbf{ArgMax}: Select the class label that \emph{maximizes} the current noisy posterior distribution, \( j = \arg \max p_\theta(\rc_t|\rvy) \).
    \item \textbf{Sampling}: Sample the class label from the current noisy posterior distribution, \( j \sim p_\theta(\rc_t|\rvy) \).
    \item \textbf{ArgMin}: Select the class label that \emph{minimizes} the current noisy posterior distribution, \( j = \arg \min p_\theta(\rc_t|\rvy) \).
\end{enumerate}

Table~\ref{tab:selection} reports the performance of DiDiCM-CP under these three label selection methods. 
For each, we evaluate nine uncertainty conditions, varying image corruption and data scarcity levels, consistent with the experimental setup presented in Section~\ref{sec:exp}.

Interestingly, across most uncertainty conditions, the ArgMin strategy yields the best performance. 
We hypothesize that this approach, which selects the least confident class label under the current noisy posterior \( p_\theta(\rc_t|\rvy) \), appears to encourage the model to more thoroughly explore the space of possible classes and to refine its decision boundaries based on perceptual features rather than prior confidence.

The Sampling strategy performs second best, occasionally surpassing ArgMin in certain moderate-uncertainty scenarios, but showing less stability under high-uncertainty conditions, when only 25\% of the training data is available. 

Finally, the ArgMax strategy consistently performs worst. 
We attribute this to the bias it introduces: by repeatedly reinforcing the most confident class label, the model may overfit to its initial predictions and fail to adequately re-examine alternative interpretations of the input image.

\section{Visualizing DiDiCM Using Misclassified Images}
\label{app:bad}

In this section, we visualize the top-5 class probabilities for randomly selected images on which both DiDiCM and the standard classifier produce incorrect Top-1 predictions. The corresponding illustrations are provided in Figure~\ref{fig:bad}.

Figure~\ref{fig:bad} presents 10 randomly sampled images in which both the classification approaches bring incorrect prediction. For each image, we display the sorted top-5 class probabilities obtained from DiDiCM (based on DiDiRN-50) and from a standard classifier (based on ResNet-50). Both models were trained using our Strong Aug training recipe, detailed in Appendix~\ref{app:training}, which represents the state-of-the-art training procedure for ResNet-50~\cite{wightman2021resnet}.

We observe that in most cases, the standard classifier exhibits overconfident Top-1 probability estimates for the incorrect class labels. In contrast, DiDiCM produces a more balanced probability distribution across plausible class options. For instance, in image \#8, depicting a spider web with a car mirror in the background, the true label is "spider web". While the standard classifier assigns nearly 100\% confidence to the "car mirror" label, DiDiCM appropriately allocates substantial probability mass to the "spider web" class as well, reflecting a more calibrated representation of the posterior distribution.

\section{Empirical Study of DiDiRN Model Size}
\label{app:all_resnet}

In this section, we provide additional experimental results evaluating the performance of DiDiCM compared to standard classification models across various model sizes, leveraging the DiDiRN architectural design based on ResNet models.

The proposed DiDiRN architecture is built on top of the ResNet framework to enable direct comparison with corresponding baseline models, as described in Section~\ref{sec:didirn} and illustrated in Figure~\ref{fig:didirn}.

Table~\ref{tab:all_resnet} summarizes the full results, while Figure~\ref{fig:all_resnet} visualizes the corresponding accuracy gains. We focus on the most challenging uncertainty scenario from our empirical study in Section~\ref{sec:exp}, in which the input image resolution is 56 and only 25\% of the data are available for training. Under this setting, we train DiDiRN models corresponding to standard ResNet sizes, i.e., DiDiRN-18 through DiDiRN-152.
We then evaluate and compare the Top-1 and Top-5 accuracies of the DiDiRN variants with their baseline ResNet counterparts.

\begin{table}[t]
\centering
\small
\begin{tabular}{c|c|c||cc}
\toprule
\textbf{Method} & \textbf{Model} & \shortstack{\textbf{Params}\\\textbf{(M)}} & \shortstack{\textbf{Top-1}\\(\%)} & \shortstack{\textbf{Top-5}\\(\%)} \\
\midrule
\multirow{5}{*}{\shortstack{Standard\\Classifiers}}
                         & ResNet-18   & 11.7 & 43.46 & 67.70 \\
                         & ResNet-34   & 21.8  & 47.41 & 71.65 \\
                         & ResNet-50   & 25.6  & 53.77 & 76.03 \\
                         & ResNet-101  & 44.5  & 54.33 & 76.35 \\
                         & ResNet-152  & 60.2  & 54.20 & 76.00 \\
\midrule
\multirow{5}{*}{\shortstack{\textbf{Ours:}\\\textbf{DiDiCM-CP}\\$1/\Delta t = 8$}}  
                         & DiDiRN-18   & 12.1  & \textbf{49.13} & \textbf{74.01} \\
                         & DiDiRN-34   & 22.6  & \textbf{52.50} & \textbf{77.09} \\
                         & DiDiRN-50   & 27.7  & \textbf{59.05} & \textbf{81.93} \\
                         & DiDiRN-101  & 49.0  & \textbf{61.86} & \textbf{83.95} \\
                         & DiDiRN-152  & 66.6  & \textbf{62.95} & \textbf{84.61} \\

\bottomrule
\end{tabular}
\caption{\textbf{ImageNet Top-1 and Top-5 accuracy across model sizes.}
Comparison of DiDiCM-CP (8 steps) and standard classifiers at input resolution 56, trained on 25\% of the available data. DiDiCM consistently outperforms standard classifiers across all model sizes.}
\label{tab:all_resnet}
\end{table}

\begin{figure}[t]
    \centering
    \includegraphics[width=\columnwidth]{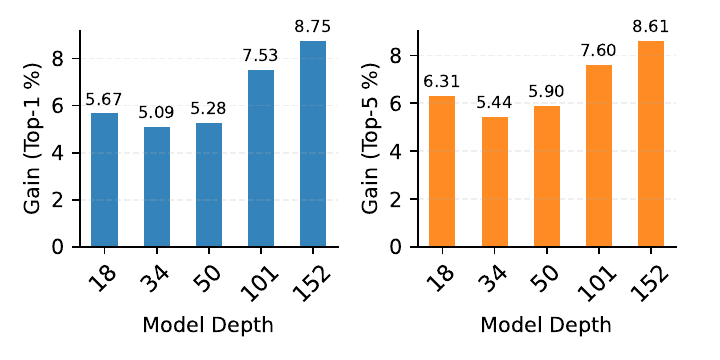}
    \caption{\textbf{ImageNet Top-1 and Top-5 accuracy gains across model sizes}. Comparison of DiDiCM-CP (8-step) and standard classifiers at 56 input resolution, trained on 25\% of the data.}
    \label{fig:all_resnet}
\end{figure}

\begin{table}[t]
\centering
\small
\begin{tabular}{l||cc}
\toprule
\textbf{Hyperparameter} & \textbf{Weak Aug} & \textbf{Strong Aug} \\
\midrule
\multicolumn{3}{c}{\textit{Optimization}} \\
\midrule
Epochs                & 600 & 600 \\
Warmup epochs         & 5 & 5 \\
Batch size            & 1024 & 2048 \\
Optimizer             & LAMB & LAMB \\
Learning rate         & 2$\times$10$^{-3}$ & 3.5$\times$10$^{-3}$ \\
Weight decay          & 0.1 & 0.1 \\
Mixed precision       & v & v \\
\midrule
\multicolumn{3}{c}{\textit{Image Augmentation}} \\
\midrule
Repeated augmentation & - & 3 \\
Horizontal flip       & v & v \\
Random resized crop   & v & v \\
RandAugment           & - & 7 / 0.5 \\
Color jitter          & 0.4 & - \\
Mixup $\alpha$        & - & 0.2 \\
CutMix $\alpha$       & - & 1.0 \\
Label smoothing $\epsilon$ & - & 0.1 \\
\midrule
\multicolumn{3}{c}{\textit{Evaluation}} \\
\midrule
Test crop ratio       & 0.875 & 0.875 \\
\bottomrule
\end{tabular}
\caption{Details of the two training recipes used for the DiDiCM experiments. Both are variants of the A1 recipe from ResNet-SB \cite{wightman2021resnet}. \textbf{Weak Aug} incorporates standard image augmentations as implemented in PyTorch \cite{paszke2019pytorch}, while \textbf{Strong Aug} follows the original A1 procedure without modifications.}
\label{tab:aug}
\end{table}

As shown in Table~\ref{tab:all_resnet}, DiDiRN consistently outperforms standard classifiers across all model sizes. Figure~\ref{fig:all_resnet} further indicates that accuracy gains in Top-1 and Top-5 metrics are approximately consistent for each model scale.
The largest observed improvement occurs with DiDiRN-152, which achieves an accuracy gain of $\sim$8.7\% over ResNet-152 in both Top-1 and Top-5 metrics. This finding suggests that scaling up DiDiRN architectures has the potential to yield higher performance than conventional classification models.

\begin{figure*}[t]
    \centering
    \includegraphics[width=\textwidth]{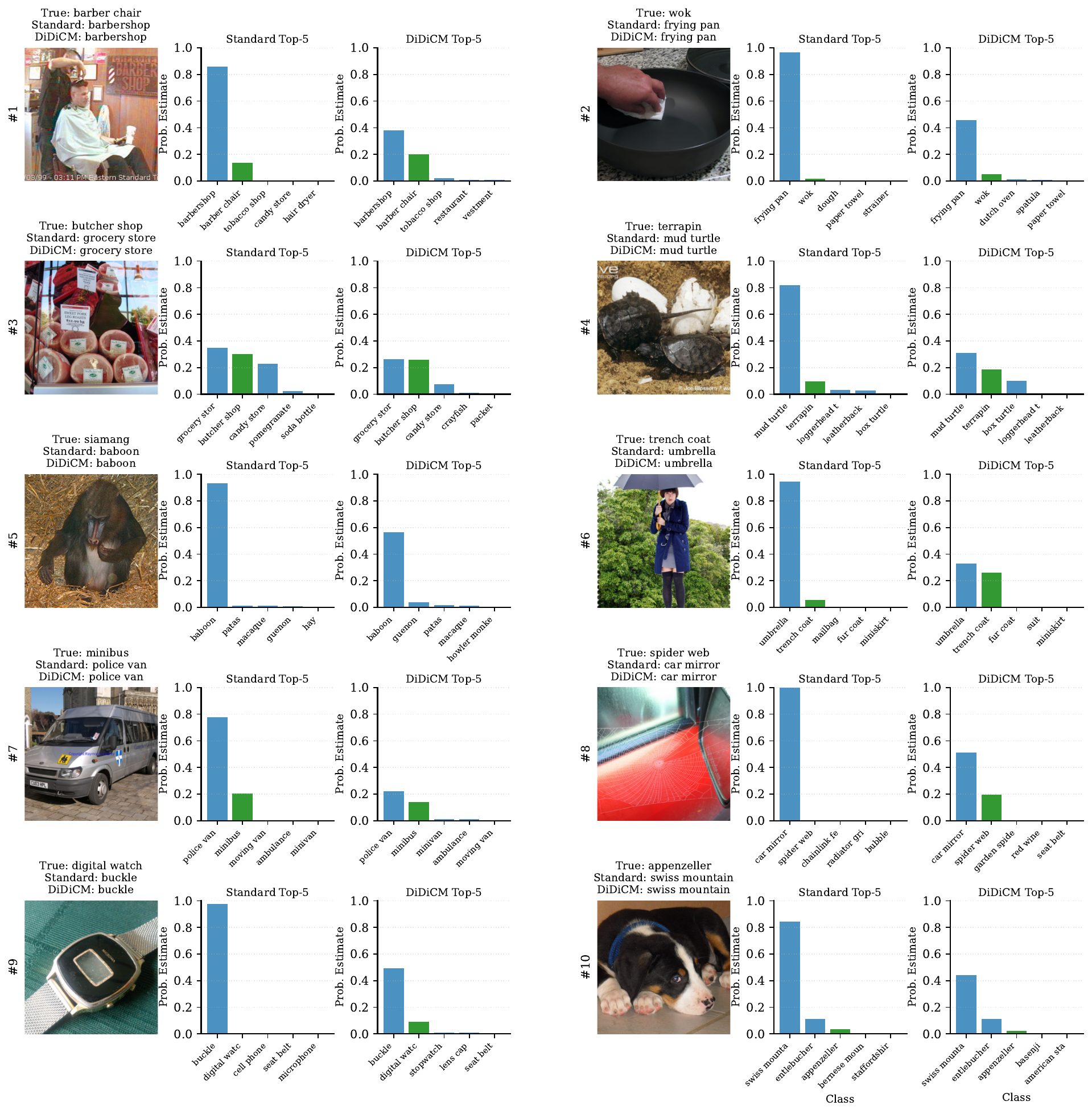}
    \caption{Random images of misclassifications by both DiDiCM (DiDiRN-50) and the standard classifier (ResNet-50), trained with the Strong Aug recipe. The results indicate that while the standard classifier tends toward overconfidence, DiDiCM yields more balanced prediction probabilities.}
    \label{fig:bad}
\end{figure*}

\section{Technical Training Details}
\label{app:training}

In this section we provide additional technical details of all the experiments in this paper.
In Table \ref{tab:aug}, we summarize the training details relevant only to our DiDiCM experiments.

\textbf{Architecture}. ~
Most of our experiments are conducted using ResNet-50. For DiDiCM, we construct the corresponding DiDiRN-50 model. In Appendix~\ref{app:all_resnet}, we extend our empirical study with additional experiments across all standard ResNet model sizes.

\textbf{Dataset}. ~
All experiments were conducted on the challenging ImageNet-1k (ILSVRC-2012) dataset \cite{russakovsky2015imagenet}.
Image normalization was applied consistently during both training and evaluation.
For training, we applied the image augmentations described below. For evaluation, we used the standard center crop ratio of 0.875.

\textbf{Image Augmentations}. ~
We evaluate model performance under two augmentation regimes.
The first, referred to as \emph{Weak Aug}, consists of the standard image augmentations used for the ImageNet dataset as implemented in PyTorch \cite{paszke2019pytorch}.
These include random horizontal flipping, random resized cropping, and ColorJitter.
The second, denoted by \emph{Strong Aug}, employs a more advanced, state-of-the-art augmentation pipeline for ResNet models, following ResNet-SB \cite{wightman2021resnet}.
It includes RandAugment \cite{cubuk2020randaugment} with repeated augmentation enabled, Mixup \cite{zhang2017mixup} and CutMix \cite{yun2019cutmix} with label smoothing, and the standard random horizontal flipping and random resized cropping.

\begin{table*}[htb]
\centering
\small
\begin{tabular}{c|c|c|c||ccc|ccc|ccc}
\toprule
\multirow{2}{*}{\textbf{Method}} & \multirow{2}{*}{\textbf{NFEs}} & \multirow{2}{*}{\textbf{Memory}} & \multirow{2}{*}{\textbf{Metric}} & \multicolumn{3}{c|}{\textbf{Training Ratio - 25\%}} & \multicolumn{3}{c|}{\textbf{Training Ratio - 50\%}} & \multicolumn{3}{c}{\textbf{Training Ratio - 100\%}} \\
 &  &  &  & 56 & 112 & 224 & 56 & 112 & 224 & 56 & 112 & 224 \\
\midrule
\multirow{2}{*}{\shortstack{Standard\\Classifiers}} 
 & \multirow{2}{*}{1} & \multirow{2}{*}{-} & Top-1 
& 53.77 & 66.51 & 71.85 & 60.51 & 72.31 & \textbf{77.05} & 64.71 & 75.96 & \textbf{80.42} \\
 &&& Top-5 & 76.03 & 85.41 & 88.72 & 81.74 & 89.81 & 92.64 & 85.32 & 92.30 & 94.60 \\
\multirow{2}{*}{DiDiCM-CL} 
 & \multirow{2}{*}{$N\frac{1}{\Delta t}$} & \multirow{2}{*}{$\mathcal{O}\!(K$+$N)$} & \cellcolor{blue!10}Top-1 & \cellcolor{blue!10}56.96 & \cellcolor{blue!10}67.15 & \cellcolor{blue!10}70.78 & \cellcolor{blue!10}64.51 & \cellcolor{blue!10}72.58 & \cellcolor{blue!10}76.20 & \cellcolor{blue!10}68.74 & \cellcolor{blue!10}76.86 & \cellcolor{blue!10}79.77 \\
 &  &  & \cellcolor{blue!10}Top-5 & \cellcolor{blue!10}72.27 & \cellcolor{blue!10}81.61 & \cellcolor{blue!10}83.83 & \cellcolor{blue!10}79.57 & \cellcolor{blue!10}85.60 & \cellcolor{blue!10}88.58 & \cellcolor{blue!10}83.17 & \cellcolor{blue!10}88.89 & \cellcolor{blue!10}91.51 \\
 
\multirow{2}{*}{DiDiCM-CP} 
 & \multirow{2}{*}{$\frac{1}{\Delta t}$} & \multirow{2}{*}{$\mathcal{O}\!(K^2)$} & \cellcolor{green!10}Top-1 & \cellcolor{green!10}\textbf{58.95} & \cellcolor{green!10}\textbf{68.71} & \cellcolor{green!10}\textbf{72.10} & \cellcolor{green!10}\textbf{65.57} & \cellcolor{green!10}\textbf{73.57} & \cellcolor{green!10}76.78 & \cellcolor{green!10}\textbf{69.84} & \cellcolor{green!10}\textbf{77.63} & \cellcolor{green!10}80.15 \\
 &  &  & \cellcolor{green!10}Top-5 & \cellcolor{green!10}\textbf{81.80} & \cellcolor{green!10}\textbf{88.64} & \cellcolor{green!10}\textbf{90.89} & \cellcolor{green!10}\textbf{86.60} & \cellcolor{green!10}\textbf{91.57} & \cellcolor{green!10}\textbf{93.56} & \cellcolor{green!10}\textbf{89.21} & \cellcolor{green!10}\textbf{93.81} & \cellcolor{green!10}\textbf{95.23} \\
\bottomrule
\end{tabular}
\caption{\textbf{Quality-Efficiency Analysis}.
ImageNet top-1 and top-5 accuracy of DiDiCM under efficient configurations ($1/\Delta t = 2$, $N = 16$ for DiDiCM-CL). Results compare DiDiCM-CP (2 NFEs) and DiDiCM-CL (32 NFEs) against standard classification (1 NFE).}
\label{tab:nfe}
\end{table*}

\textbf{Optimization}. ~
All models were trained using the LAMB optimizer \cite{you2019large} with a cosine annealing learning rate schedule and a weight decay of 0.1. Training was conducted for 600 epochs, with the first 5 epochs designated for linear warmup.
For experiments employing Weak Aug, we used a batch size of 1024 and a base learning rate of 2$\times$10$^{-3}$. The learning rate was scaled according to the training data ratio using the square-root scaling rule, $\text{lr} = \text{base} \times \sqrt{1/\text{ratio}}$, to mitigate overfitting.
For standard classification experiments with Strong Aug, we followed ResNet-SB \cite{wightman2021resnet} by using a batch size of 2048 and a base learning rate of 5$\times$10$^{-3}$. The learning rate was scaled according to the same policy described above.
For DiDiCM experiments with Strong Aug, we also used a batch size of 2048 but with a base learning rate of 3.5$\times$10$^{-3}$. For these experiments, learning rate scaling based on the data ratio was omitted, as we found it to be non-beneficial for performance in contrast for standard classifiers performance.

\section{The Quality-Efficiency Tradeoff}
\label{app:nfe}

In this section, we evaluate the top-1 and top-5 accuracy of DiDiCM under efficient configuration settings, considering both DiDiCM-CP and DiDiCM-CL. Specifically, we set the diffusion steps to $\frac{1}{\Delta t} = 2$ and, for DiDiCM-CL, the number of class labels to be averaged to $N = 16$. These settings correspond to 2 NFEs for DiDiCM-CP and 32 NFEs for DiDiCM-CL. Notably, DiDiCM-CL is more memory-efficient than DiDiCM-CP, providing a viable alternative in resource-constrained scenarios. Full results are reported in Table~\ref{tab:nfe}.

Consistent with our observations using 8 diffusion steps in Table~\ref{tab:main_results}, DiDiCM-CP consistently outperforms standard classifiers in top-5 accuracy, with performance gains increasing under higher uncertainty. For top-1 accuracy, DiDiCM also surpasses standard classifiers in settings that challenge uncertainty estimation. In low-uncertainty conditions (resolution 224, full and half training data), standard classifiers slightly outperform DiDiCM configured under this efficient setting.

While DiDiCM-CL is more memory-efficient than DiDiCM-CP, we find that its efficient configuration with 32 NFEs achieves performance comparable to DiDiCM-CP, with an accuracy difference of $\sim$1\%. This suggests that DiDiCM-CL can serve as a practical, memory-conscious alternative without substantial sacrifice in accuracy.

\end{document}